\documentclass{article}

\usepackage[final]{nips_2016}

\usepackage[utf8]{inputenc} 
\usepackage[T1]{fontenc}    
\usepackage{hyperref}       
\usepackage{url}            
\usepackage{booktabs}       
\usepackage{amsfonts}       
\usepackage{amsmath}
\usepackage{nicefrac}       
\usepackage{microtype}      
\usepackage{xcolor}
\usepackage{graphicx}

\newcommand{\bx}{\mathbf{x}}
\newcommand{\btheta}{\mathbf{\theta}}
\newcommand{\bz}{\mathbf{z}}
\newcommand{\bp}{\mathbf{p}}
\newcommand{\bc}{\mathbf{c}}
\newcommand{\br}{\mathbf{r}}
\newcommand{\bC}{\mathbf{C}}
\newcommand{\by}{\mathbf{y}}
\newcommand{\bM}{\mathbf{M}}
\newcommand{\bI}{\mathbf{I}}
\newcommand{\bV}{\mathbf{V}}

\newcommand{\bw}{\mathbf{w}}
\newcommand*\diff{\mathop{}\!\mathrm{d}}
\newcommand{\cN}{\mathcal{N}}

\newcommand{\bv}{\mathbf{v}}

\newcommand{\bh}{\mathbf{h}}

\newcommand{\cD}{\mathcal{D}}
\newcommand{\cB}{\mathcal{B}}

\definecolor{dgreen}{rgb}{0,.7,0}
\definecolor{dyellow}{rgb}{.7,.7,0}
\definecolor{dred}{rgb}{.7,0,0}
\definecolor{dblue}{rgb}{0,0,0.7}

\newtheorem{theorem}{Theorem}[section]

\newtheorem{proposition}[theorem]{Proposition}

\newenvironment{proof}[1][Proof]{\begin{trivlist}
\item[\hskip \labelsep {\bfseries #1}]}{\end{trivlist}}

\title{Asynchronous Stochastic Gradient MCMC with Elastic Coupling}

\author{
  Jost Tobias Springenberg \;\;  Aaron Klein \;\; Stefan Falkner \;\; Frank Hutter   \\
  Department of Computer Science, University of Freiburg\\
  \texttt{\{springj,kleinaa,sfalkner,fh\}@cs.uni-freiburg.de} \\
}

\begin{document}

\maketitle

\begin{abstract}
  We consider parallel asynchronous Markov Chain Monte Carlo (MCMC)
  sampling for problems where we can leverage (stochastic) gradients
  to define continuous dynamics which explore the target
  distribution. We outline a solution strategy for this setting
  based on stochastic gradient Hamiltonian Monte Carlo sampling
  (SGHMC) which we alter to include an elastic coupling term
  that ties together multiple MCMC instances. The proposed strategy turns
  inherently sequential HMC algorithms into asynchronous parallel
  versions. First experiments empirically show that the resulting parallel sampler
  significantly speeds up exploration of the target distribution, when
  compared to standard SGHMC, and is less prone to the harmful effects
  of stale gradients than a naive parallelization approach.
\end{abstract}

\section{Introduction and Background}
Over the last years the ever increasing complexity of machine learning (ML)
models together with the increasing amount of data that is available to train them has resulted in a
great demand for parallel training and inference algorithms that can
be deployed over many machines. 
To meet this demand, ML practitioners have increasingly relied on
(asynchronous) parallel stochastic gradient descent (SGD) methods
for optimizing model parameters \citep{Recht2011,Dean2012,Zhang2015}. In contrast to this, the
literature on efficient parallel methods for sampling from the posterior over model parameters given some data is much more scarce. Examples of algorithms for this setting typically are constrained to specific model classes \citep{Ahmed2012,Ahn2015,Simsekli2015} or ``only'' consider
data parallelism \citep{AhnSW14}; in summary -- with the
exception of recent work by \citet{chen2016} -- a general sampling pendant
of asynchronous SGD methods is missing.
 
In this paper we consider the general problem of sampling from an arbitrary
posterior distribution over parameters $\btheta \in \mathbb{R}^{n}$ given
a set of observed datapoints $\bx \in \cD$ where we have $K$ machines
available to speed up the sampling process. Without loss of generality
we can write the mentioned posterior as $p(\btheta | \cD) \propto \exp(-U(\btheta))$
where we define $U(\btheta)$ to be \emph{the potential energy} $U(\btheta) = - \sum_{\bx \in D} \log
p(\bx | \btheta)  - \log p(\btheta)$. Many algorithms for solving this problem
such as Hamiltonian Monte Carlo \citep{duane-1987,neal2010} further augment this potential energy
with terms depending on auxiliary variables $\by \in \mathbb{R}^A$ to speed up sampling. In
this more general case we write the posterior as $p(\bz | \cD) \propto
\exp(-H(\bz))$, where $\bz = \lbrack \btheta, \by \rbrack \in \mathbb{R}^m$ denotes the
collection of all variables and $H(\bz)$ denotes the \emph{Hamiltonian}
$H(\bz) = H(\btheta,\by) = U(\btheta) + g(\btheta, \by)$. Samples from $p(\btheta |
\cD)$ can then be obtained by sampling from $p(\bz | \cD)$ and
discarding $\by$ if we additionally assume that marginalizing out
$\by$ from $p(\bz | \cD)$ results only in a constant offset $c$; that is
we require $\int_{\by \in \mathbb{R}^A} \exp(-g(\btheta,\by)) d\by = c$.

\subsection{Stochastic Gradient Hamiltonian Monte Carlo}
\label{sect:sgmcmc}
Stochastic gradient MCMC (SGMCMC) algorithms solve the above described sampling
problem by assuming access to a -- possibly noisy -- gradient of the
potential $U(\btheta)$ with respect to parameters $\btheta$. In this case --
assuming properly chosen auxiliary variables -- sampling can be
performed via an analogy to physics by simulating a system based on
the Hamiltonian $H(\bz)$. More precisely, following the general
formulation from \citet{ma-nips15}, one can simulate a stochastic
differential equation (SDE) of the form 
\begin{equation}
  d\bz = f(\bz) \diff t + \sqrt{2 D(\bz)} \diff \bw_t,
  \label{eq:diff_sgmcmc}
\end{equation}
where $f(\bz) : \mathbb{R}^m \rightarrow \mathbb{R}^m$ denotes the
deterministic drift incurred by $H(\bz)$, $D(\bz) : \mathbb{R}^m
\rightarrow \mathbb{R}^m \times \mathbb{R}^m$ is a diffusion
matrix, the square root is applied element-wise, and $\bw_t \in \mathbb{R}^m$ denotes Brownian motion. Under
fairly general assumptions on the functions $f$ and $D$ the unique stationary
distribution of this system is equivalent to the posterior distribution $p(\btheta |
\cD)$. Specifically, \citet{ma-nips15} showed that if $f(\bz)$ is of the following specialized form (in which we are free to choose $D(\bz)$ and $Q(\bz)$):
\begin{equation}
  f(\bz) = - \big ( D(\bz) + Q(\bz) \big ) \begin{pmatrix} \nabla_\btheta U(\btheta) +
    \nabla_\btheta g(\btheta, \by) \\ \nabla_y  g(\btheta, \by) 
  \end{pmatrix} + 
  \Gamma(\bz), \quad \Gamma_i(\bz) = \sum_{j=1}^m
  \frac{\partial}{\partial \bz_j} \left ( D_{i,j}(\bz) + Q_{i,j}(\bz) \right),
  \label{eq:specialized_diff}
\end{equation} 
then the stationary distribution is equivalent to the posterior if
$D(\bz)$ is positive semi-definite and $Q(\bz)$ is
skew-symmetric. Importantly, this holds also if only noisy
estimates $\nabla_\btheta \tilde{U}(\btheta)$ of the gradient
$\nabla_\btheta U(\btheta)$ computed on a randomly sampled subset of the data are
available (as in our case).

\subsubsection{Practical SGMCMC implementations}
\label{sect:sghmc}
In practice, for any choice of $H(\bz)$, $D(\bz)$ and $Q(\bz)$,
simulating the differential equation \eqref{eq:diff_sgmcmc} on a
digital computer involves two approximations. First, the SDE is
simulated in discretized steps resulting in the update rule
\begin{equation}
  \bz_{t+1} =\bz_{t} - \epsilon_t \Big[ \big ( D(\bz) + Q(\bz) \big ) \begin{pmatrix} \nabla_\btheta U(\btheta) -
    \nabla_\btheta g(\btheta, \by) \\ \nabla_y  g(\btheta, \by) 
  \end{pmatrix} - \Gamma(\bz) \Big ] + \cN \big( 0, 2\epsilon_t
  D(\bz_t) \big),
\label{eq:sgmcmc_updates}
\end{equation}
where we slightly abuse notation and take $\cN \big( \mu, \Sigma)$ to denote the addition of \emph{a sample from} an m-dimensional multivariate Gaussian distribution.
Second, when dealing with large datasets, exact computation of the
gradient $\nabla_\btheta U(\btheta)$ becomes computationally prohibitive and one
thus relies on a stochastic approximation computed on a randomly
sampled subset of the data: $\nabla_\btheta \tilde{U}(\btheta)$ with $
\tilde{U}(\bz) = \frac{N}{|\mathcal{B}|} \sum_{\bx \in \mathcal{B}}
\log p(\bx | \btheta)  - \log p(\btheta)$ where $\mathcal{B} \subset
\cD$. The stochastic gradient $\nabla_\btheta \tilde{U}(\btheta)$ is then Gaussian
distributed with some variance $V$; leading to the noise
term in the above described SDE.
Using these two approximations one can derive the following discretized system of equations for a stochastic gradient variant of Hamiltonian Monte Carlo
\begin{equation}
 \begin{aligned}
   \btheta_{t+1} &= \btheta_{t} + \epsilon \bM^{-1} \bp_t \\
   \bp_{t+1} &= \bp_{t} -\epsilon \nabla_\btheta \tilde{U}(\btheta_t) -
   \epsilon \bV \bM^{-1} \bp_t + \mathcal{N}(0, 2 \epsilon \bV),
 \end{aligned}
 \label{eq:sghmc}
\end{equation}
which, following \citet{ma-nips15}, can be seen as an instance of Equation \eqref{eq:sgmcmc_updates}
with $\by = \bp$, $g(\btheta, \bp) = \bp^T \bM^{-1} \bp$ and where $D(\bz) = \begin{pmatrix} 
  0 & 0 \\
  0 & \bV
                                     \end{pmatrix}$ and $Q(\bz) = \begin{pmatrix} 
  0 & \bI \\
  -\bI & \bV
                                     \end{pmatrix}$.

\section{Parallelization schemes for SG-MCMC}
\label{sect:parallel_schemes}
We now show how one can utilize the computational power of $K$ machines to speed up a given sampling procedure relying on the dynamics described in Equations \eqref{eq:sghmc}.
As mentioned before, the update equations derived in Section \ref{sect:sghmc} involve alternating updates to both the parameters and the uauxiliary variables, leading to an inherently sequential algorithm.
If we now want to utilize $K$ machines to
speed up sampling we thus face the non-trivial challenge of parallelizing these updates.
In general we can identify two solutions\footnote{We note that if the computation of the stochastic gradient $\nabla_\btheta \tilde{U}(\btheta)$ is based on a large number of data-points (or if we want to reduce the variance of our estimate by increasing $|\cB|$) we could potentially spread out the data-set $\cD$ over the $K$ machines allowing us to parallelize computation without the need for asynchronous updating. While this is an interesting problem in its own right and there already exists a considerable amount of literature on running parallel MCMC over sub-sets of data \cite{Scott2016,Rabi2015,Neiswanger14} we here focus on parallelization schemes that do not make this assumption because of their broad applicability.}:

\textbf{I)} For a naive parallelization strategy we can send the variables $\btheta$ to $K$ different machines every $s$ steps from a parameter server. Each machine then computes a gradient estimate for the current step $\nabla_\btheta \tilde{U}(\tilde{\btheta}^k_t)$ (note that we only approximately have $\tilde{\btheta}^k_t \approx \btheta_t$ due to the communication period $s$, i.e. $\tilde{\btheta}^k_t$ at each machine might be a stale parameter). The server then waits for $O$ gradient estimates $\nabla_\btheta \tilde{U}(\tilde{\btheta}^k_t)$ to be sent back and simulates the system from Eq. \eqref{eq:diff_sgmcmc} using $\nabla_\btheta \tilde{U}(\btheta_t) \approx \frac{1}{O} \sum_{k=1}^O \nabla_\btheta \tilde{U}(\tilde{\btheta}^k_t)$;

\textbf{II)} We can set-up $K$ MCMC chains (one per machine) which independently update a parameter vector $\bz^k$ (where $k \in [1, K]$ denotes the machine), following the dynamics from Eq. \eqref{eq:sghmc}.

While the second approach clearly results in Markov chains that asymptotically sample from the correct distribution -- and might result in a more diverse set of samples than a simulation using a single chain -- it is also clear that it cannot speed up convergence of the individual chains (our desired goal) as there is no interaction between them.
The first approach, on the other hand, is harder to analyze. We can observe that if $s = 1$ and we wait for $O = K$ gradient estimates in each step we obtain an SG-MCMC algorithm with parallel gradient estimates but synchronous updates of the dynamic equations. Consequently, such a setup preserves the guarantees of standard SG-MCMC but \emph{requires synchronization between all machines in each step}. In a real-world experiment (where we might have heterogeneous machines and communication delays) this will result in a large communication overhead. For choices of $s > 1$ and $O < K$ -- the regime we are interested in -- we cannot rely on the standard convergence analysis for SG-MCMC anymore. Nonetheless, if we concentrate on the analysis of $s > 1$ and $O = 1$ (i.e. completely asynchronous updates) we can interpret the stale parameters to simply result in more noisy estimates of $\nabla_\btheta \tilde{U}(\btheta_t)$ that can be used within the dynamic equations from Eq. \ref{eq:sghmc}. The efficacy of such a parallelization scheme then intuitively depends on the amount of additional noise introduced by the stale parameters and requires a new convergence analysis. During the preparation of this manuscript, concurrent work on SGMCMC with stale gradients derived a theoretical foundation for this intuition \citep{chen2016}.
Interestingly, we will in the following empirically show that the additional noise is unproblematic for small $s$ in the range $1 < s < 4$ (for which a convergence speed-up with naive parallelization can therefore be achieved, but which result in a large communication overhead in distributed systems) but becomes problematic with growing $s$ . 
We believe these results are in accordance with the mentioned recent work by \citet{chen2016}, yet a unification of their theory with our proposed new algorithm remains as important future work.

\section{Stochastic Gradient MCMC with Elastic Coupling}
\label{sect:method}
Given the negative analysis from Section \ref{sect:parallel_schemes} one might wonder whether approach \textbf{II)} (the idea of running $K$ parallel MCMC chains) can be altered such that the $K$ chains are only loosely coupled; allowing for faster convergence while avoiding excessive communication. To this end we propose to consider the following alternative parallelization scheme:

\textbf{IIa)} To speed up $K$ SGMCMC chains we couple the $K$ parameter vectors through an additional center variable $\bc$ to which they are elastically attached. We collect updates to this center variable at a central server and broadcast an updated version of it every $s$ steps across all machines. We note that an approach based on this idea was recently utilized to derive an asynchronous SGD optimizer in \citet{Zhang2015}, serving as our main inspiration. A discussion of the connection between their deterministic and our stochastic dynamics is presented in Section \ref{sect:easgd}.

To derive an asynchronous SGMCMC variant with $K$ samplers and elastic coupling -- as described above -- we consider
an augmented Hamiltonian with $\bz = \lbrack \btheta^1, \dots, \btheta^K,
\bp^1, \dots, \bp^K, \bc, \br \rbrack$:
\begin{equation}
 H(\bz) = \sum_{i=1}^K \Big ( U(\btheta^i) + {\bp^i}^T \bM^{-1} \bp^i \Big
 ) + \frac{1}{K} \sum_{i=1}^K \frac{\alpha}{2} \| \btheta^i - \bc \|^2_2 + {\br}^T \bM^{-1} \br,
 \label{eq:hamiltonian_ec}
\end{equation}
where we can interpret $\bc$ as a centering mass (with momentum $\br$)
through which the asynchronous samplers are elastically coupled and $\alpha$ specifies the coupling strength. It is easy to see that for $\alpha = 0$ we can decompose the sum from Eq. \eqref{eq:hamiltonian_ec} into $K$ independent terms which each constitute a Hamiltonian corresponding to the standard SGHMC case (and we thus recover the setup of $K$ independent MCMC chains). Further, for $\alpha > 0$ we obtain a joint, altered, Hamiltonian in which -- as desired -- the $K$ parameter vectors are elastically coupled. Simulating a system according to this Hamiltonian exactly would again result in an algorithm requiring synchronization in each simulation step (since the change in momentum for all parameters depends on the position of the center variable and thus, implicitly, on all other parameters).

If we, however, assume only a noisy measurement of the center variable and its momentum is available in each step we can derive a mostly asynchronous algorithm for updating each $\theta^i$. To achieve this let us assume we store our current estimate for the center variable and its momentum at a central server. This server receives updates for $\bc$ and $\br$ from each of the $K$ samplers every $s$ steps and replies with their current values. Assuming a Gaussian normal distribution on the error of this current value each sampler then keeps track of a noisy estimate $\tilde{\bc} \approx \bc$  of the center variable which is used for simulating updates derived from the Hamiltonian in Equation \eqref{eq:hamiltonian_ec}. 
From these assumptions we derive the following discretized dynamical equations:
\begin{equation}
 \begin{aligned}
   \btheta^i_{t+1} &= \btheta^i_{t} + \epsilon \bM^{-1} \bp^i_t, \\
   \bc_{t+1} &= \bc_{t} + \epsilon \bM^{-1} \br_t, \\
   \bp^i_{t+1} &= \bp^i_{t} -\epsilon \nabla \tilde{U}(\btheta^i_t) -
   \epsilon \bV \bM^{-1} \bp^i_t - \epsilon \alpha (\btheta^i_t - \bc_t)  +
   \mathcal{N}(0, 2 \epsilon^2 (\bV + \bC)), \\
   \br_{t+1} &= \br_{t} -
   \epsilon \bC \bM^{-1} \br_t - \epsilon \alpha
   \frac{1}{K} \sum_{i=1}^K (\bc_t - \btheta^i_t)  + \mathcal{N}(0, 2 \epsilon^2 \bC),
 \end{aligned}
 \label{eq:ec_sgmcmc}
\end{equation}
where $\bV$ specifies the noise due to stochastic gradient estimates and $\bC$ is the variance of the aforementioned noisy center variable. As before, we use the notation $\mathcal{N}(\mu, \Sigma)$ to refer to a sample from a Gaussian distribution with mean $\mu$ and covariance $\Sigma$. We note that, while the presented dynamical equations were derived from SGHMC, the elastic coupling idea does not depend on the basic Hamiltonian from Equation \eqref{eq:sghmc}. We can thus derive similar asynchronous samplers for any SGMCMC variant including first order stochastic Langevin dynamics \citep{WelTeh2011a} or any of the more advanced techniques reviewed in \citet{ma-nips15}.

When inspecting the Equations \eqref{eq:ec_sgmcmc} we can observe that, similar to approach \textbf{I}, they also contain an additional noise source: noise is injected into the system due to potential staleness of the center variable. However, since this noise only indirectly affects the parameters $\btheta^i$ one might hope that the center variable acts as a buffer, damping the injected noise. If this were the case, we would expect the resulting algorithm to be more robust to communication delays than the naive parallelization approach. In addition to this consideration, the proposed dynamical equations have a convenient form which makes it easy to verify that they fulfill the conditions for a valid SGMCMC procedure.

\begin{proposition}
The dynamics of the system from Eq. \eqref{eq:ec_sgmcmc} has the posterior distribution
$p(\btheta \mid \cD)$ as the stationary
distribution for all $K$ samplers.
\end{proposition}
\begin{proof}
To show this, we first establish that Equations \eqref{eq:ec_sgmcmc} correspond to a discretized dynamics following the general form given by Equations \eqref {eq:diff_sgmcmc} and \eqref{eq:specialized_diff} where $D(\bz) = \text{diag}([0, \bV, 0, \bC])$ 
and 
$Q(\bz) = \begin{pmatrix} 
   A & B \\
   B & A \end{pmatrix}$, with $A = \begin{pmatrix} 0 & \bI \\
-\bI & 0 \end{pmatrix}$ and $B = \begin{pmatrix} 0 & 0 \\
0 & 0 \end{pmatrix}$, thus fulfilling the requirements outlined in Section \ref{sect:sgmcmc}. Furthermore, to see that marginalization of the auxiliary variables results in a constent offset, we can first identify $g(\btheta^i,\by) = \nicefrac{\alpha}{2K} \| \btheta^i - \bc \|^2_2 + {\br}^T \bM^{-1} \br + {\bp^i}^T \bM^{-1} \bp^i$, with $\by = [\bc, \br, \bp]$. Solving the integral $\int_{\by \in \mathbb{R}^A} \exp(-g(\btheta^i,\by)) d\by$ then amounts to evaluating Gaussian integrals and is therefore easily checked to be constant, as required. Thus, simulating the the dynamical equations results in samples from $p(\btheta, \by \mid \cD)$ and discarding the auxilliary variables $\by$ gives the desired result.
\end{proof}
\begin{figure}[!h]
  \vspace{-0.3cm}
  \centering
  \includegraphics[width=0.49\textwidth]{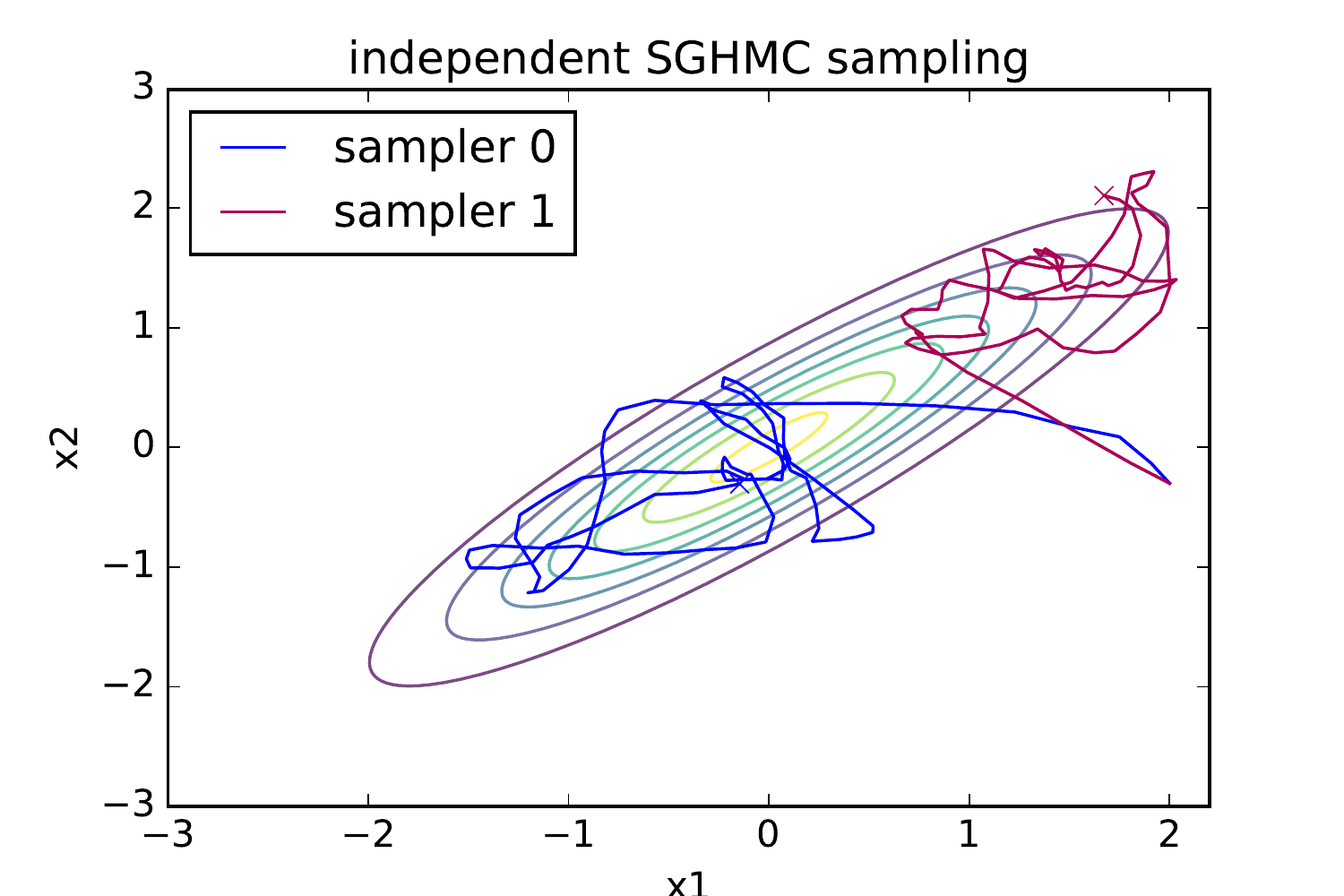}
  \includegraphics[width=0.49\textwidth]{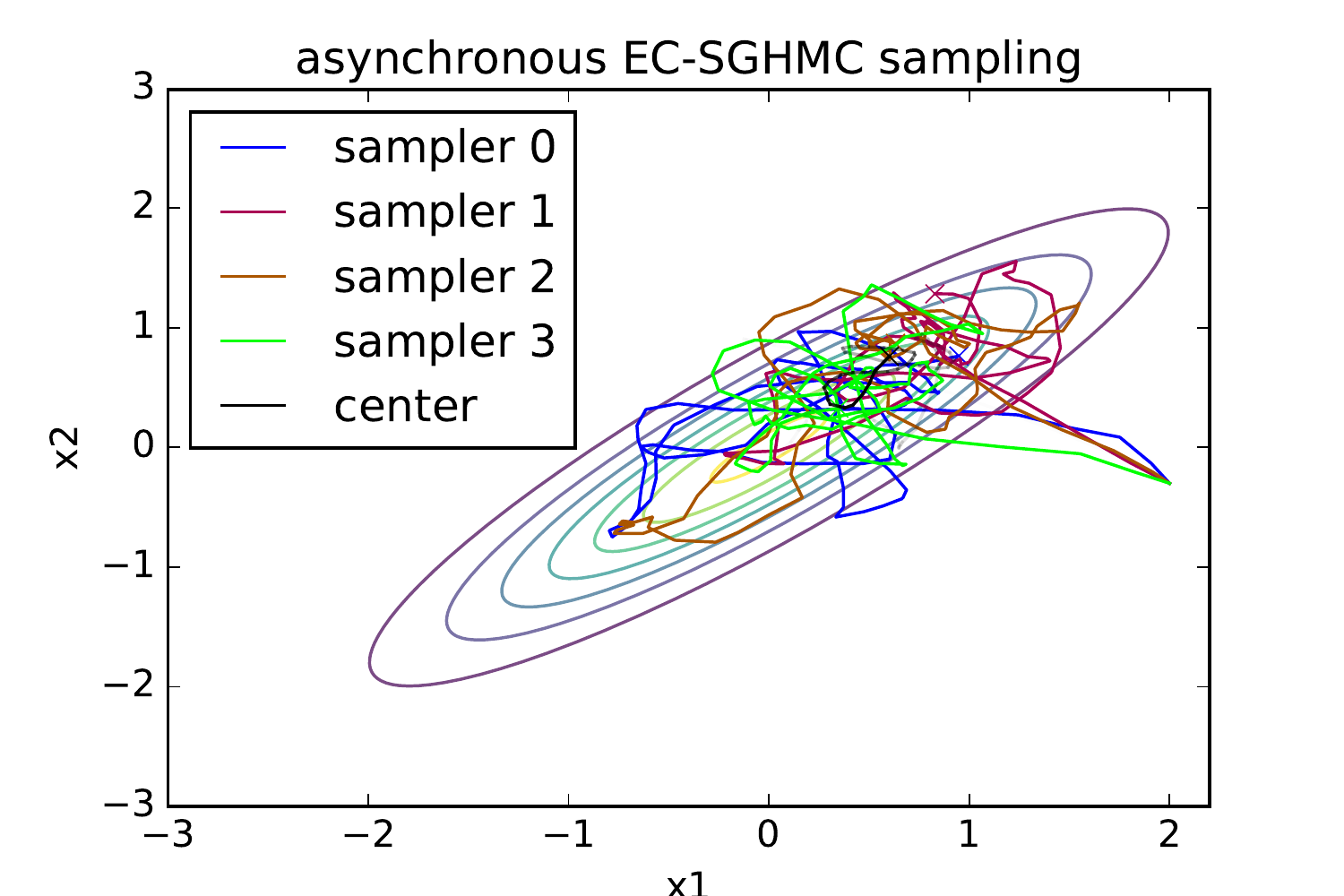}
\caption{Comparison between the first 100 sampling steps performed by SGHMC and the elastically coupled SGHMC variant for sampling from a simple two dimensional Gaussian distribution. An animated video of the samplers can be found at \url{https://goo.gl/ZZv1fG}.}
\label{fig:toy_example}
  \vspace{-0.2cm}
\end{figure}

\section{First Experiments}
We designed an initial set of three experiments to validate our sampler. First, to get an intuition for the behavior of the elastic coupling term we tested the sampler on a low dimensional toy example. In Figure \ref{fig:toy_example} we show the first 100 steps taken by both standard SGHMC (left) and our elastically coupled sampler (EC-SGHMC) with four parallel threads (right) when sampling from a two-dimensional Gaussian distribution (starting from the same initial guess). The hyperparameters were set to $\alpha = 1$, $\epsilon = 1e-2$, $\bC = \bV = \mathbf{I}$. We observe that two independent runs of SGHMC take fairly different initial paths and, depending on the noise, it can happen that SGHMC only explores low-density regions of the distribution in its first steps (cf. purple curve). In contrast the four samplers with elastic coupling quickly sample from high density regions and show coherent behaviour.

As a second experiment, we compare EC-SGHMC to standard SGHMC and the naive parallelization described in Section \ref{sect:parallel_schemes} (Async SGHMC in the following). We use each method for sampling the weights of a two layer fully connected neural network (800 units per layer, ReLU activations, Gaussian prior on the weights, batch size $100$) that is applied to classifying the MNIST dataset. For our purposes, we can interpret the neural network as a function that  parameterizes a probability distribution over classes
\begin{equation}
  p(y = i | \bx, \theta) \propto \exp(f_i(\bx, \theta)),
\end{equation}
where $y$ is the label of the given image and $f(\bx, \theta)$ is the output vector of the neural network with parameters $\theta$. We place a Gaussian prior on the network weights $p(\theta) \propto exp(\lambda \| \theta \|^2_2)$ (where we chose $\lambda = 10^{-5}$) and sample from the posterior
\begin{equation}
  p(\theta | \mathcal{D}) \propto \prod_{(\bx^j,y^j) \in \mathcal{D}} p(y^j | \bx^j, \theta) p(\theta),
  \label{eq:nn_post}
\end{equation}
for a given dataset $\mathcal{D}$. The results of this experiment are depicted in Figure \ref{fig:nets} (left). We plot the negative log likelihood over time and observe that both parallel samplers (using $K=6$ parallel threads) perform significantly better than standard SGHMC. However, when we increase the communication delay and only synchronize threads every $s = 8$ steps the additional noise injected into the sampler via this procedure becomes problematic for Async SGHMC whereas EC-SGHMC copes with it much more gracefully.

\begin{figure}[!h]
\vspace{-0.2cm}
\centering
\includegraphics[width=0.49\textwidth]{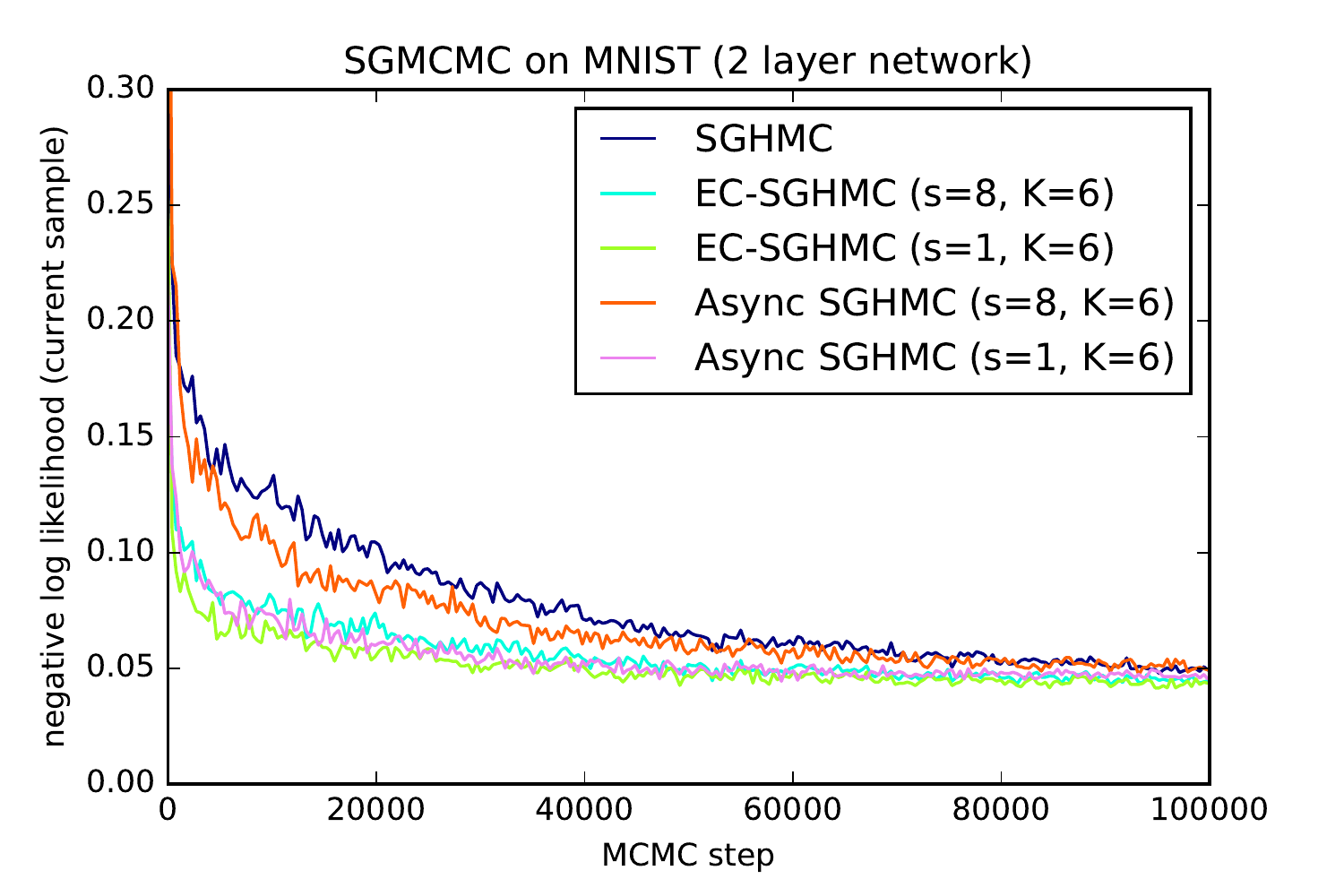}
\includegraphics[width=0.49\textwidth]{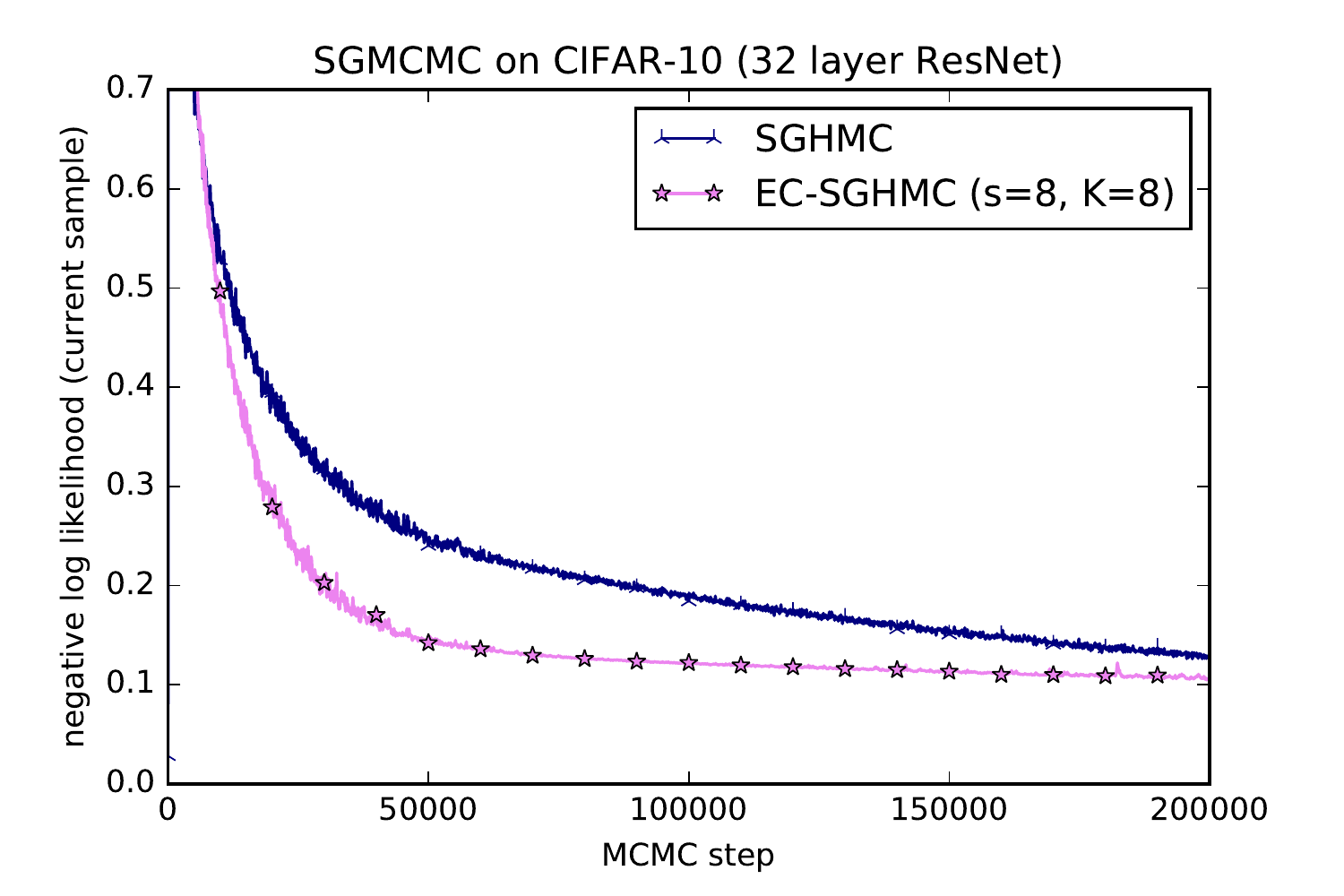}
\caption{Comparison between different SGMCMC samplers for sampling from the posterior over neural network weights for a fully connected network on MNIST (left) and a residual network on CIFAR-10 (right). Best viewed in color.}
\label{fig:nets}
\end{figure}

Finally, to test the scalability of our approach, we sampled the weights of a 32-layer residual net applied to the CIFAR-10 dataset. We again aim to sample the posterior given by Equation \eqref{eq:nn_post} only now the neural network $f(\bx, \theta)$ is the 32-layer residual network described in \citet{he-16cvpr} (with batch-normalization removed). The results of this experiment are depicted in Figure \ref{fig:nets} (right) showing that, again, EC-SGHMC leads to a significant speed-up over standard SGHMC sampling.

\section{Connection to Elastic Averaging SGD}
\label{sect:easgd}
As described in Section \ref{sect:method} an elastic coupling technique similar to the one used in this manuscript was recently used to accelerate asynchronous stochastic gradient descent \citep{Zhang2015}. Although the purpose of our paper is not to derive a new SGD method -- we instead aim to arrive at a scalable MCMC sampler -- it is instructive to take a closer look at the connection between the two methods.

To establish this connection we aim to re-derive the elastic averaging SGD (EASGD) method from \citet{Zhang2015} as the deterministic limit of the dynamical equations from \eqref{eq:ec_sgmcmc}. Removing the added noise from these equations, setting $\bM$ to the identity matrix, and performing the variable substitutions $\bv^i = \epsilon \bM p^i$, $\bh^i = \epsilon \bM r^i$, $\xi = \bV = \bC$ yields the dynamical equations
\begin{equation}
 \begin{aligned}
   \btheta^i_{t+1} &= \btheta^i_{t} + \bv^i_t, \\
   \bc_{t+1} &= \bc_{t} + \bh_t, \\
   \bv^i_{t+1} &= \bp^i_{t} -\epsilon \nabla \tilde{U}(\btheta^i_t) -
   \xi \bv^i_t - \epsilon \alpha (\btheta^i_t - \bc_t), \\
   \bh_{t+1} &= \br_{t} - \xi \bh_t - \epsilon \alpha
   \frac{1}{K} \sum_{i=1}^K (\bc_t - \btheta^i_t).
 \end{aligned}
 \label{eq:sgmcmc_eamsgd}
\end{equation}
In comparison, re-writing the updates for the EASGD variant with momentum (EAMSGD) from \citet{Zhang2015} in our notation (and replacing the Nesterov momentum with a standard momentum) we obtain
\begin{equation}
 \begin{aligned}
   \btheta^i_{t+1} &= \btheta^i_{t} + \bv^i_t - \epsilon \alpha (\btheta^i_t - \bc_t), \\
   \bc_{t+1} &= \bc_{t} - \epsilon \alpha
   \frac{1}{K} \sum_{i=1}^K (\bc_t - \btheta^i_t), \\
   \bv^i_{t+1} &= \bp^i_{t} -\epsilon \nabla \tilde{U}(\btheta^i_t) -
   \xi \bv^i_t,
 \end{aligned}
 \label{eq:eamsgd}
\end{equation}
where, additionally, \citet{Zhang2015} propose to only update $\bc_{t+1}$ every $s$ steps and drop the terms including $\bc_{t+1}$ in all update equations in intermittent steps.

As expected, at first glance the two sets of update equations look very similar. Interestingly, they do however differ with respect to integration of the elastic coupling term and the updates to the center variable: In the EAMSGD Equations \eqref{eq:eamsgd} the center variables are not augmented with a momentum term and the elastic coupling force influences the parameter values $\theta$ directly rather than, indirectly, through their momentum $\bp$.

From the physics perspective that we adopt in this paper these updates are thus ``wrong'' in the sense that they break the interpretation of the variables $\theta$, $\bc$ and $\bp$ as generalized coordinates and generalized momenta. It should be noted that there also is no straight-forward way to recover a valid SGMCMC sampler corresponding to a stochastic variant of Equations \eqref{eq:eamsgd} from the Hamiltonian given in Equation \eqref{eq:hamiltonian_ec}. Our derivation thus suggests alternative update equations for EAMSGD. An interesting avenue for future experiments thus is to thoroughly compare the deterministic updates from Equations \eqref{eq:sgmcmc_eamsgd} with the EAMSGD updates both in terms of empirical performance and with respect to their convergence properties. An initial test we performed suggests that the former perform at least as good as EAMSGD. We also note that EASGD without momentum can exactly be recovered as the deterministic limit of our approach (without the above described discrepancies) if we were to randomly re-sample the auxilliary momentum variables in each step -- and would thus simulate stochastic gradient Langevin dynamics \cite{WelTeh2011a}.

\section{Conclusion}
In this manuscript we have considered the problem of parallel asynchronous MCMC sampling with stochastic gradients. We have introduced a new algorithm for this problem based on the idea of elastically coupling multiple SGMCMC chains. First experiments suggest that the proposed method compares favorably to a naive parallelization strategy but additional experiments are required to paint a conclusive picture. We have further discussed the connection between our method and the recently proposed stochastic averaging SGD (EASGD) optimizer from \citet{Zhang2015}, revealing an alternative variant of EASGD with momentum.

\bibliographystyle{unsrtnat}
{
\small
\bibliography{async_sgmcmc}
}

\end{document}